\documentclass{article}

\usepackage{iclr2027_conference,times}
\iclrfinalcopy 
\usepackage[utf8]{inputenc} 
\usepackage[T1]{fontenc}    
\usepackage{hyperref}       
\usepackage{url}            
\usepackage{booktabs}       
\usepackage{nicefrac}       
\usepackage{microtype}      
\usepackage{xcolor}         

\usepackage{algorithm}    
\usepackage{algorithmic}  

\usepackage{amsmath,amsfonts,bm}
\usepackage{amssymb}
\usepackage{mathtools}
\usepackage{amsthm}
\newtheorem{theorem}{Theorem} 
\title{Time Delay System for Sequence Modeling}

\author{
\textbf{Youshan Zhang}\textsuperscript{1}\thanks{This research was funded by the research
project of Chuzhou University (Grant No. 2025qd36).}, and 
\textbf{Jialu Li}\textsuperscript{2}\thanks{Corresponding author.} \\
\textsuperscript{1}School of Artificial Intelligence, Chuzhou University, Chuzhou, China\\
\textsuperscript{2}Department of Electrical and Computer Engineering, Boston University, Boston, USA\\[3pt]
{\small \texttt{youshan\_zhang@chzu.edu.cn, jli21@bu.edu}}
}

\begin{document}

\maketitle

\begin{abstract}

Existing sequence models share a critical limitation: they lack an explicit, interpretable mechanism for content-aware temporal recall. LSTMs compress everything into a single hidden state, thereby losing information about specific past events. Transformers attend to all tokens indiscriminately at quadratic cost, while state-space models like Mamba compress history into a fixed-dimensional state without explicit content-addressable memory. To address this gap, we introduce a novel framework grounded in nonlinear time-delay systems, in which the effective delay is dynamically estimated based on the semantic similarity between current and historical inputs. Our approach features three key innovations: (i) a differentiable semantic lag estimator that computes content-aware soft attention weights over history, producing a weighted historical context and similarity score; (ii) a diagonal state space with closed-form cumulative sum solution, achieving O(1) per-step complexity with no sequential loops; and (iii) an interpretable fuzzy semantic modulation mechanism that learns context-dependent rules for balancing instantaneous processing versus historical recall, where rule centers correspond to prototypical semantic contexts. Experimental results on text classification and generation outperform state-of-the-art methods.
\end{abstract}

\section{Introduction}
Existing sequence modeling architectures, such as Long Short-Term Memory (LSTM)~\citep{sherstinsky2020fundamentals}, transformers, and Mamba~\citep{gu2023mamba}, although achieving promising performance, share a limitation, that is, a lack of an explicit and interpretable mechanism for temporal recall. This limitation motivates our research on sequence modeling using nonlinear time-delay systems (TDS). In real-world sequential modeling problems, the relevance of historical information depends not only on temporal distance but also on semantic similarity between current and past inputs. Therefore, we propose a semantic delay-driven sequence modeling framework that explicitly estimates adaptive temporal recall based on semantic relevance. Instead of propagating through a single hidden state, as in recurrent neural networks (RNNs), or attending to all tokens, as in transformers, our proposed model dynamically determines which historical states influence the current computation more via differentiable semantic delay estimation and interpretable fuzzy semantic modulation.

Our approach presents three key contributions:
\begin{itemize}
    \item First, we propose a differentiable semantic lag estimator that computes contextual similarity scores over historical states and produces weighted semantic historical context through soft attention delay aggregation. This provides explicit semantic temporal recall and interpretable delay estimation.  
    \item Second, we design a hidden dynamics using a diagonal state-space structure with a closed-form cumulative sum solution, enabling fully parallelizable O(1) complexity per step without sequential recurrence or quadratic attention computation.
    \item Third, we introduce an interpretable fuzzy semantic modulation mechanism that adaptively balances current input processing and delayed historical recall using learnable fuzzy rules, providing interpretable semantic control.
\end{itemize}

\section{Related Work}
Sequence modeling for NLP enables models to process and generate text by capturing temporal dependencies. Architectures such as recurrent neural networks and state-space models have been proposed to improve the modeling of temporal dependencies~\citep{sherstinsky2020fundamentals}. 

\textbf{Recurrent Neural Networks and LSTM}. Early sequence modeling approaches primarily relied on RNNs, which operate via a hidden state that evolves over time. However, they suffer from vanishing and exploding gradient problems, limiting their ability to capture long-range dependencies~\citep{bengio1994learning}. To address this, the LSTM~\citep{hochreiter1997long} incorporates gating mechanisms to regulate information flow, improving its ability to model long-term dependencies. However, a core limitation of LSTM is that it captures temporal differences implicitly, without an explicit representation of how past information is delayed or propagated over time, as all time steps are compressed into the hidden state~\citep{pascanu2013difficulty}, making it difficult to interpret in human terms. This lack of explicit temporal semantics limits its ability to model human-like reasoning over the relationship between past contextual information and the present. 
Therefore, while LSTM can retain relevant information for longer sequences, it does not provide an interpretable mechanism for understanding how certain past inputs influence the current state. 

\textbf{State Space Models and Efficient Sequence Modeling}. State space models (SSMs) have recently become a powerful framework for sequence modeling~\citep{rangapuram2018deep,smith2208simplified,zhou2023deep}. The classic state-space model is formulated as a latent state that evolves through structured linear transitions, enabling a continuous-time interpretation of sequences~\citep{kalman1960new,kailath1980linear}. In recent years, the Structured State Space sequence model (S4) was proposed by Gu et al.~(\citeyear{gu2021efficiently}), which leveraged HiPPO-based projections to improve the model's ability to process longer-range sequences while remaining computationally efficient. Later, the Mamba architecture proposed by Gu and Dao~(\citeyear{gu2023mamba}) introduced selective state space models, where input-dependent gating facilitates state transitions, enabling linear-time complexity while maintaining strong performance across long sequences. 
Despite these advancements, the SSM-based models primarily achieved computational and structural improvements rather than conceptual reasoning. Specifically, the content representation of structured state transitions remains largely obscured. The learned transition matrices and input-dependent parameters~\citep{gu2021efficiently} do not explicitly encode semantically meaningful notions such as temporal delay, lag structure, or the relevance of past information, leaving the temporal dynamics still lacking interpretability. 

\textbf{Time-Delay Modeling in Dynamical Systems}. Time-delay systems have been studied in control theory and dynamical systems to model temporal dependencies with explicit lags~\citep{insperger2011semi,hale2013introduction}. In delay differential equations (DDEs), the system's evolution depends on past states with a defined time delay, enabling explicit modeling of the relationship between temporal memories independent of position in time ~\citep{smith2011introduction,waibel2013phoneme,snyder2015time}. Recent work continues to investigate the delay structures in complex dynamical systems. Peddinti et al.~\citep{peddinti2015time} first introduced the Time Delay Neural Network (TDNN) for efficiently modeling long-range temporal dependencies in speech recognition while avoiding the high training cost of recurrent models. Huang et al.~(\citeyear{huang2019deeper}) proposed deeper TDNNs for acoustic modeling by increasing the depth of temporal context layers. Ji et al.~(\citeyear{ji2024trainable}) further improved TDNNs by introducing trainable and continuous time delays as learnable parameters, allowing the network to directly identify both the magnitude of delays and the underlying nonlinear dynamics from data. 
However, these approaches remain limited in the context of NLP. Existing TDNN-based models are primarily evaluated on low-dimensional time-series data and have not been widely scaled to high-dimensional sequence modeling tasks such as language understanding or generation. Additionally, traditional time delay networks does not explicitly incorporate semantic meanings to moderate temporal dynamics~\citep{jatnika2019word2vec,sivakumar2020review}, revealing the insufficient semantic awareness and interpretability of recent NLP model designs.

From the review, we can see that current sequence modeling approaches lack an explicit, interpretable mechanism for modeling how semantic information propagates over time with variable delay. RNNs and LSTMs encode temporal dependencies implicitly through hidden states and gating mechanisms. State-space models improve efficiency but remain black boxes. Interpretability methods are largely post-hoc and do not reflect internal temporal processes. Time-delay systems and fuzzy logic offer interpretability but are not well-integrated with modern semantic representations in NLP. This gap motivates the exploration of new approaches that can integrate semantic understanding, temporal delay, and interpretability. Our proposed semantic delay-driven fuzzy state update model addresses these limitations by integrating semantic delay estimation with fuzzy logic-based control, enabling both effective and interpretable sequence modeling.

\begin{figure}
    \centering
    \includegraphics[width=1.\linewidth]{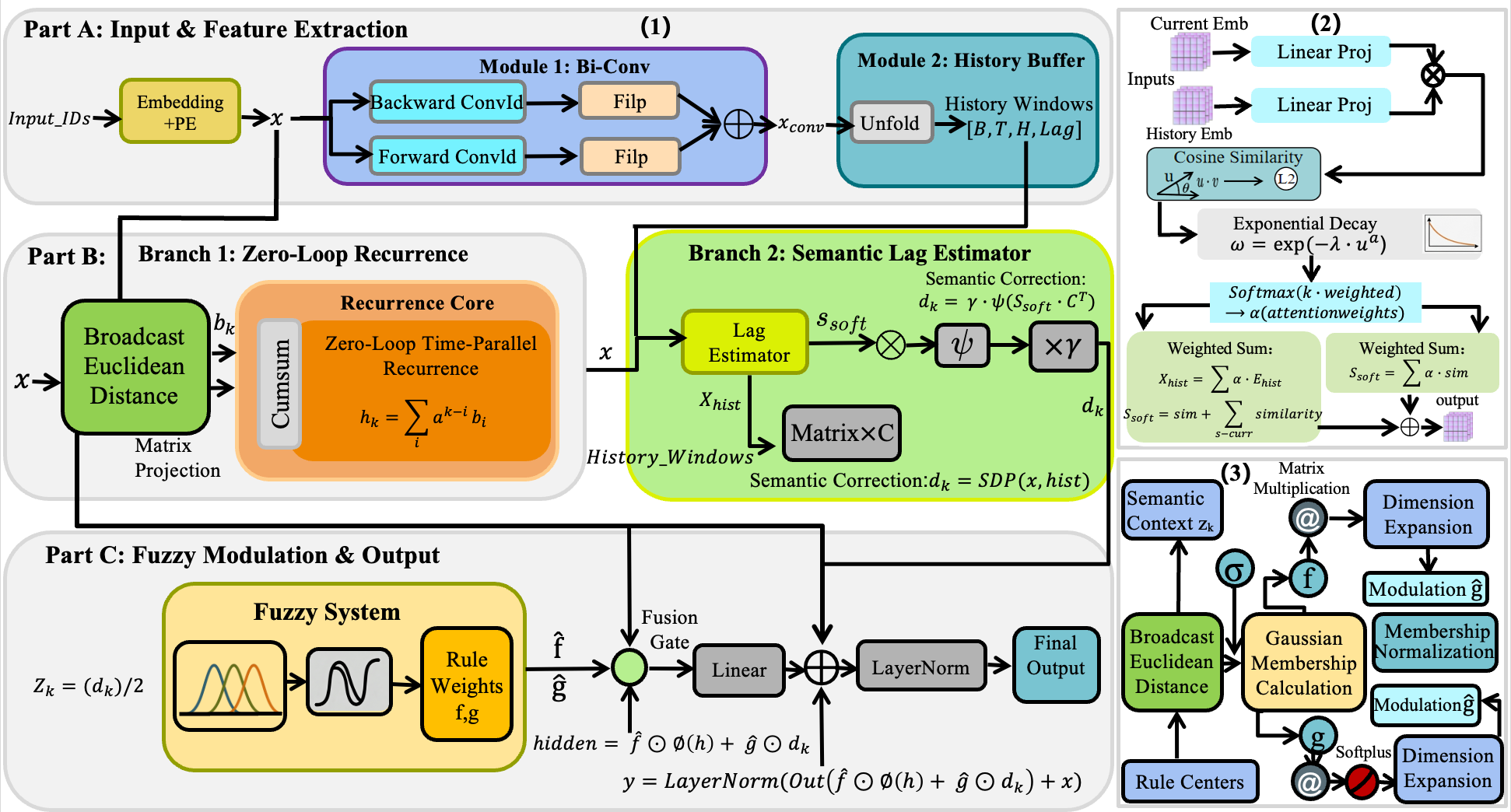}
    \caption{ (1) is the overall architecture of our proposed TDS model. (2) is the semantic lag estimator, and (3) is the fuzzy system.  }
    \vspace{-0.3cm}
    \label{fig:placeholder}
\end{figure}
\section{Methods}
\textbf{Motivation}. Existing sequence models struggle with a fundamental problem: they cannot explicitly decide when to look back in history based on semantic relevance. LSTMs compress everything into a single hidden state, losing specific past events. Transformers look at everything but at quadratic cost---inefficient and unfocused. Mamba improves efficiency but still lacks an interpretable mechanism for content-aware recall. Our time delay system directly addresses this by introducing an explicit semantic lag estimator that answers: "Which past token is most relevant to what I'm seeing now?" This allows the model to dynamically retrieve and integrate specific historical information---much like how humans recall relevant memories when triggered by current input. Combined with interpretable fuzzy rules that learn when to trust the present versus the past, our approach achieves both the expressiveness of attention and the efficiency of linear recurrence, while offering a clear, mathematical foundation grounded in delay differential equations.

\subsection{Problem Definition}

In conventional recurrent or state-space language models, the evolution of the current hidden state $\bm{s}(t)$, (where $t$ is the current time) is assumed to depend only on the current input $\bm{x}(t)$ and the previous state $\bm{s}(t-1)$.
However, natural language often exhibits \textbf{semantic delay} $\tau(t)$: words or phrases can influence
the meaning of future tokens after a variable time lag.  
To capture this, we introduce a \textbf{nonlinear time-delay dynamical system} as the foundation for sequence modeling.

\subsection{Mathematical Foundation}
Formally, let $\bm{s}(t) \in \mathbb{R}^d$ denote the current state ($d$ is the dimension of the state) and $\bm{x}(t) \in \mathbb{R}^m$ the current input ($e$.$g$., token embedding, $m$ is the dimension of the input). The system evolves according to a nonlinear time-delay differential equation:
\begin{equation}
    \dot{\bm{s}}(t)
    = \bm{f}\big(\bm{s}(t), \bm{s}(t-\tau(t)), \bm{x}(t)\big),
    \label{eq:nonlinear_delay_system}
\end{equation}  
where $\dot{\bm{s}}(t)$ is the change rate of the current state $s(t)$; $\tau(t) \in [0, L]$ denotes a time delay that can vary dynamically with the semantic context ($L$ is the maximum of $\tau$ ($\tau_{\max}$));  $\bm{s}(t-\tau(t))$ is delayed past state; and $\bm{f}$ is a nonlinear function that can simultaneously process the current state, delayed past state and current input.


\subsection{Functional Separation}
According to the decomposition theory of multivariate functions, any smooth multivariate vector function $\bm{f}$ can be decomposed at the reference state point $\bm{s}_0$ as following,
\begin{equation}\label{eq:f_sep}
\bm{f}\big(\bm{s}, \bm{s}_\tau, \bm{x}\big) =    \bm{f}\big(\bm{s}, \bm{s}_0, \bm{x}\big) + [\bm{f}\big(\bm{s}, \bm{s}_\tau, \bm{x}\big) - \bm{f}\big(\bm{s}, \bm{s}_0, \bm{x}\big)]
\end{equation}
where $\bm{s}_\tau = \bm{s}(t-\tau(t))$.
We can define two componentfunctions based on the above split:

\begin{itemize}
    \item Instantaneous Function: $F(s,x) = \bm{f}\big(\bm{s}, \bm{s}_0, \bm{x}\big)$, which captures the system's behavior when the delayed state is at its reference value.
    \item Delay Function: $G\big(\bm{s}, \bm{s}_\tau, \bm{x}\big) =  \bm{f}\big(\bm{s}, \bm{s}_\tau, \bm{x}\big) - \bm{f}\big(\bm{s}, \bm{s}_0, \bm{x}\big) $, which captures the deviation caused by the delayed state being different from the reference. 
\end{itemize}

Substituting these two definitions back to Eq.~\eqref{eq:f_sep}, we can get
\begin{equation}\label{eq:s_inter}
\dot{\bm{s}}(t)
    = F\big(\bm{s}(t), \bm{x}(t)\big) +   G\big(\bm{s}(t),  \bm{s}(t-\tau(t)), \bm{x}(t)\big). 
\end{equation}
Eq.~\eqref{eq:s_inter} is still complex since $G$ depends on both the current state $\bm{s}(t)$ and the delayed state $\bm{s}(t-\tau(t))$. For a more practical and interpretable model, we introduce a physically-motivated simplification: the primary function of the delay is to inject information from the past history, thus, we approximate the delay function as being independent of the current state as follows:
\begin{equation}
 G\big(\bm{s}(t),  \bm{s}(t-\tau(t)), \bm{x}(t)\big) \approx \tilde{G}(\bm{s}(t-\tau(t)), \bm{x}(t-\tau(t)))  
\end{equation}
Hence, we can rewrite Eq.~\eqref{eq:s_inter} as:
\begin{equation}\label{eq:s_dem}
\dot{\bm{s}}(t)
    = F\big(\bm{s}(t), \bm{x}(t)\big) +   G(\bm{s}(t-\tau(t)), \bm{x}(t-\tau(t)))  
\end{equation}
where the first term $F(\cdot)$ captures the instantaneous state evolution driven by the current input and the second term $G(\cdot)$ introduces delayed feedback that depends on past states and inputs. We drop the tilde on $G$ for notation simplicity. Eq.~\eqref{eq:s_dem} is the fundamental definition of a continuous nonlinear time-delay system.
It subsumes the standard state-space model when $\tau(t)\!=\!0$ and $G\!\equiv\!0$.

\subsection{Structure Parameterization}
To obtain a learnable form for neural network modeling, we introduce structured linear mappings and nonlinear activations as follows:
\begin{align}
F(\mathbf{s}(t),\mathbf{x}(t))  
&= \phi\!\big(\mathbf{A}\mathbf{s}(t) + \mathbf{B}\mathbf{x}(t)\big), \label{eq:F_term}\\
G(\mathbf{s}(t-\tau(t)),\mathbf{x}(t-\tau(t)))
&= \gamma\,\psi\!\big(\hat{s}(t)\,\mathbf{C}\,\bar{\mathbf{x}}(t-\tau(t))\big), \label{eq:G_term}
\end{align} 
where
$\mathbf{A},\mathbf{B},\mathbf{C}\in\mathbb{R}^{d\times d}$ are learnable matrices,
$\phi(\cdot)$ and $\psi(\cdot)$ are nonlinear activation functions (e.g., SiLU, ReLU),
$\gamma$ is a scalar weight controlling the influence of the delayed term,
and $\hat{s}(t)\in[0,1]$ is the semantic similarity between the current and delayed embeddings.
Substituting Eq.~\eqref{eq:F_term} and Eq.~\eqref{eq:G_term} into Eq.~\eqref{eq:s_dem} yields:
\begin{equation}
\dot{\mathbf{s}}(t)
= \phi\!\big(\mathbf{A}\mathbf{s}(t) + \mathbf{B}\mathbf{x}(t)\big)
+ \gamma\,\psi\!\big(\hat{s}(t)\,\mathbf{C}\,\mathbf{x}(t-\tau(t))\big).
\label{eq:nonlinear_delay_dde}
\end{equation}
This equation describes a nonlinear dynamical system with both instantaneous and delayed semantic interactions.

\subsection{Discretization}

For sequential data (\textit{e}.\textit{g}., text ), the system evolves in discrete time steps.
Let the time discretization be $t_k = k\Delta t$,
and approximate the time derivative in Eq.~\eqref{eq:nonlinear_delay_dde}
using the first-order forward difference:
\begin{equation}
\dot{\mathbf{s}}(t_k)
\approx
\frac{\mathbf{s}_{k} - \mathbf{s}_{k-1}}{\Delta t}.
\end{equation}
Substituting this into Eq.~\eqref{eq:nonlinear_delay_dde} gives:
\begin{align}
\frac{\mathbf{s}_{k} - \mathbf{s}_{k-1}}{\Delta t}
&= \phi\!\big(\mathbf{A}\mathbf{s}_{k-1} + \mathbf{B}\mathbf{x}_{k}\big)
 + \gamma\,\psi\!\big(\hat{s}_{k}\,\mathbf{C}\,\mathbf{x}_{k-\tau_k}\big).
\end{align}
Rearranging terms leads to the discrete-time update:
\begin{equation}
\mathbf{s}_{k}
= \mathbf{s}_{k-1}
+ \Delta t\,
\Big[
\phi\!\big(\mathbf{A}\mathbf{s}_{k-1} + \mathbf{B}\mathbf{x}_{k}\big)
+ \gamma\,\psi\!\big(\hat{s}_{k}\,\mathbf{C}\,\mathbf{x}_{k-\tau_k}\big)
\Big].
\label{eq:discrete_difference}
\end{equation}

When the sampling step is normalized to $\Delta t=1$, the additive integration term can be absorbed into the nonlinear mappings, yielding a simplified discrete nonlinear delay model:
\begin{equation}
\mathbf{s}_{k}
= \phi\!\big(\mathbf{A}\mathbf{s}_{k-1} + \mathbf{B}\mathbf{x}_{k}\big)
+ \gamma\,\psi\!\big(\hat{s}_{k}\,\mathbf{C}\,\mathbf{x}_{k-\tau_k}\big),
\label{eq:discrete_state_update}
\end{equation}
which is the core discrete-time realization of our nonlinear time-delay system.

\subsection{Semantic Lag Estimator}

To calculate the second delay term $\gamma\psi\!\big(\hat{s}_{k}\,\mathbf{C}\,\mathbf{x}_{k-\tau_k}\big)$, we first need to consider the semantic delay $\tau(t)$, which is determined by the similarity between the current input and its past history. We define a semantic lag estimator that computes the similarity $\mathcal{S}$ between the current embedding $\bm{x}_t$ and each historical embedding $\bm{x}_{t-k}$ within the window, where each lag $k = 1, \ldots, L$ and  $L$ is the maximum number of semantic delay $\tau$. Here we define $\mathcal{S}$, the cosine similarity as follows:
\begin{equation}
    \mathrm{sim}_k = 
    \frac{\langle \bm{x}_t, \bm{x}_{t-k} \rangle}
         {\|\bm{x}_t\|_2 \, \|\bm{x}_{t-k}\|_2 },
    \label{eq:cosine_sim}
\end{equation}
where $||\cdot||_2$ is the L2 norm.  To suppress far-away history contexts, each similarity value is weighted by a temporal decay function $
    w_k = \exp\!\left(-\lambda k^{\alpha}\right),
$, 
where $\lambda > 0$ and $\alpha > 0$ are hyperparameters controlling the rate of exponential decay. The combined semantic activation over the past history window $\mathcal{W}_t = \{\bm{x}_{t-k}\}_{k=1}^{L}$ is defined as 
$
    \tilde{s}_{t,k} = w_k \, \mathrm{sim}_k,
$,
and the normalized semantic confidence is
\begin{equation}
    \hat{s}_t = \sigma\!\left(\sum_{k=1}^{L} \tilde{s}_{t,k}\right),
    \label{eq:s_hat}
\end{equation}
where $\sigma(\cdot)$ denotes the sigmoid function. 

\paragraph{Hard selection (interpretable).} 
The effective discrete delay is selected as the smallest lag exceeding an input-dependent threshold:
\begin{equation}
    \tau_t = \min \left\{ k \in \{1,\ldots,L\} : \tilde{s}_{t,k} \geq \theta \cdot \sigma(\|\mathbf{x}_t\|) \right\},
    \label{eq:hard_tau}
\end{equation}
with corresponding similarity score $\hat{s}_t = \tilde{s}_{t,\tau_t}$. If no lag satisfies the condition, we set $\tau_t = 0$ and $\hat{s}_t = 0$.

\paragraph{Soft selection (differentiable).} 
To enable end-to-end training, we introduce a soft attention mechanism over lags:
\begin{align}
\alpha_{t,k} &:= \frac{\exp\big(\kappa \tilde s_{t,k}\big)}{\sum_{j=1}^L \exp\big(\kappa \tilde s_{t,j}\big)}\qquad(\kappa>0),\quad
\bar x_t^{\mathrm{hist}} &:= \sum_{k=1}^L \alpha_{t,k}\, x_{t-k}, \quad
\hat s_t^{\mathrm{soft}} &:= \sum_{k=1}^L \alpha_{t,k}\, \tilde s_{t,k}.\label{eq:soft-s}
\end{align}
Then, we have:
\begin{equation}
       \tau_t^{\mathrm{soft}} = \sum_{k=1}^{L} k \cdot \alpha_{t,k}.
    \label{eq:soft_attention}
\end{equation}
Substituting these expressions into Eq.~\eqref{eq:discrete_state_update} gives the continuous and differentiable soft-delay update:
\begin{equation}
\mathbf{s}_{k} = \underbrace{\phi\!\big(\mathbf{A}\mathbf{s}_{k-1} + \mathbf{B}\mathbf{x}_k\big)}_{\text{Base Update}} + \underbrace{\gamma\,\psi\!\big(\hat{s}_{k}^{\mathrm{soft}}\,\mathbf{C}\,\bar{\mathbf{x}}_{k}^{\mathrm{hist}}\big)}_{\text{Delay Correction }},
\label{eq:core_update}
\end{equation}
with \(s_t=\bar s_t+\Delta_t^{\mathrm{soft}}\). In the limit \(\kappa\to\infty\), the soft quantities concentrate on the minimal lag $\tilde{s}_{t,k}$, and $\tau_t^{\text{soft}}$ recovers the hard selection $\tau_t$, where $\mathbf{s}_k \in \mathbb{R}^d$ is the semantic state at time $k$, $\mathbf{x}_k \in \mathbb{R}^d$ is the current token embedding, $\mathbf{A}, \mathbf{B}, \mathbf{C} \in \mathbb{R}^{d \times d}$ are learnable matrices, $\phi, \psi$ are nonlinear activation functions (e.g., SiLU, tanh), $\gamma \in \mathbb{R}$ is a global scaling parameter, $\hat{s}_{k}^{\mathrm{soft}} \in \mathbb{R}$ is the soft similarity score from the Semantic Lag Estimator, and $\bar{\mathbf{x}}_{k}^{\mathrm{hist}} \in \mathbb{R}^d$ is the soft-weighted historical context.

\begin{algorithm}[!htb]
\caption{Soft Semantic Lag Estimator}
\label{alg:soft_lag_estimator_compact}
\begin{algorithmic}[1]
\REQUIRE 
  Current embeddings $\mathbf{X}_{\text{curr}} \in \mathbb{R}^{B \times H}$ or $\mathbb{R}^{B \times T \times H}$, History windows $\mathbf{H} \in \mathbb{R}^{B \times L \times H}$ or $\mathbb{R}^{B \times T \times L \times H}$, Parameters: projection $\mathbf{W}_{\text{proj}}$, decay $\lambda, \alpha$, temperature $\kappa$
\ENSURE 
  Attention weights $\boldsymbol{\alpha}$, weighted history $\bar{\mathbf{x}}^{\text{hist}}$, soft score $\hat{s}^{\text{soft}}$

\STATE \textbf{Initialize:} $w_u = \exp(-\lambda u^{\alpha}),\ u=1,\ldots,L$ \hfill $\triangleright$ Fixed decay weights

\STATE $\mathbf{X}_{\text{curr}}^{\text{proj}} = \mathbf{W}_{\text{proj}} \mathbf{X}_{\text{curr}}$ \hfill $\triangleright$ Project to similarity space
\STATE $\mathbf{H}^{\text{proj}} = \mathbf{W}_{\text{proj}} \mathbf{H}$

\STATE $\text{sim} = \dfrac{\mathbf{X}_{\text{curr}}^{\text{proj}} \cdot \mathbf{H}^{\text{proj}^{\top}}}{\|\mathbf{X}_{\text{curr}}^{\text{proj}}\| \cdot \|\mathbf{H}^{\text{proj}}\| + \varepsilon}$ \hfill $\triangleright$ Cosine similarity

\STATE $\tilde{s}_{k,u} = \text{sim}_{k,u} \cdot w_u$ \hfill $\triangleright$ Apply temporal decay

\STATE $\alpha_{k,u} = \dfrac{\exp(\kappa \tilde{s}_{k,u})}{\sum_j \exp(\kappa \tilde{s}_{k,j})}$ \hfill $\triangleright$ Softmax over history

\STATE $\bar{\mathbf{x}}^{\text{hist}} = \sum_u \alpha_{k,u} \mathbf{x}_{k-u}$ \hfill $\triangleright$ Weighted historical context
\STATE $\hat{s}^{\text{soft}} = \sum_u \alpha_{k,u} \tilde{s}_{k,u}$ \hfill $\triangleright$ Soft similarity score

\STATE \textbf{return} $\boldsymbol{\alpha},\ \bar{\mathbf{x}}^{\text{hist}},\ \hat{s}^{\text{soft}}$
\end{algorithmic}
\end{algorithm}

\subsection{Fuzzy Semantic System}
We define the delay correction term in the core equation Eq.~\eqref{eq:core_update} as:
\begin{equation}
\mathbf{d}_k = \gamma\,\psi\!\big(\hat{s}_{k}^{\mathrm{soft}}\,\mathbf{C}\,\bar{\mathbf{x}}_{k}^{\mathrm{hist}}\big),
\label{eq:delay_correction}
\end{equation}
and the equation simplifies to:
\begin{equation}
\mathbf{s}_k = \phi\!\big(\mathbf{A}\mathbf{s}_{k-1} + \mathbf{B}\mathbf{x}_k\big) + \mathbf{d}_k.
\label{eq:core_simplified}
\end{equation}


The core equation treats both terms with equal weight (coefficient 1 for each term). However, different semantic contexts require different balances. To address this, we introduce context-dependent modulation:

\begin{equation}
\mathbf{s}_k = f(\mathbf{z}_k) \cdot \phi\!\big(\mathbf{A}\mathbf{s}_{k-1} + \mathbf{B}\mathbf{x}_k\big) + g(\mathbf{z}_k) \cdot \mathbf{d}_k,
\label{eq:modulated_update}
\end{equation}

where $\mathbf{z}_k \in \mathbb{R}^d$ is the semantic context vector, and $f, g: \mathbb{R}^d \to \mathbb{R}$ are modulation functions.

\textbf{Semantic Context Definition}. The semantic context captures both the current and recently recalled information:

\begin{equation}
\mathbf{z}_k = \frac{\mathbf{s}_{k-1} + \mathbf{d}_k}{2}.
\label{eq:semantic_context}
\end{equation}

\textbf{Fuzzy Rule-Based Parameterization}. We employ a zero-order Takagi-Sugeno-Kang (TSK) fuzzy system with $R$ rules. Each rule $r$ has a prototypical semantic context $\mathbf{c}_r \in \mathbb{R}^d$,  a consequent $f_r \in \mathbb{R}$ for base update modulation, and a consequent $g_r \in \mathbb{R}$ for delay correction modulation.

\textbf{Membership Function}. The unnormalized membership is given by a Gaussian radial basis function:

\begin{equation}
\tilde{\mu}_r(\mathbf{z}) = \exp\left(-\frac{\|\mathbf{z} - \mathbf{c}_r\|^2}{2\sigma^2}\right),
\label{eq:unnormalized_membership}
\end{equation}

where $\sigma > 0$ is the width parameter. The normalized memberships are:

\begin{equation}
\mu_r(\mathbf{z}) = \frac{\tilde{\mu}_r(\mathbf{z})}{\sum_{j=1}^{R} \tilde{\mu}_j(\mathbf{z}) + \varepsilon}, \qquad \varepsilon = 10^{-8},
\label{eq:normalized_membership}
\end{equation}

which satisfy:
$
0 \leq \mu_r(\mathbf{z}) \leq 1, \quad \sum_{r=1}^{R} \mu_r(\mathbf{z}) = 1.
$

\textbf{subsubsection}{Defuzzifications}. The modulation outputs are weighted averages of the rule consequents:

\begin{equation}
f(\mathbf{z}) = \sum_{r=1}^{R} \mu_r(\mathbf{z}) \cdot f_r, \qquad
g(\mathbf{z}) = \sum_{r=1}^{R} \mu_r(\mathbf{z}) \cdot g_r.
\label{eq:defuzzification}
\end{equation}

\textbf{Positivity Constraint}. The delay modulation is constrained to be positive via the softplus function:

\begin{equation}
\tilde{g}(\mathbf{z}) = \operatorname{softplus}\big(g(\mathbf{z})\big) = \ln\!\left(1 + e^{g(\mathbf{z})}\right) > 0.
\label{eq:positivity_constraint}
\end{equation}

\textbf{Complete Fuzzy-Modulated State Update}. Substituting all components, the complete system is:
\begin{equation}
    f_k = \sum_{r=1}^{R} \mu_r(\mathbf{z}_k) \, f_r, \quad
g_k = \ln\!\left(1 + \exp\!\left(\sum_{r=1}^{R} \mu_r(\mathbf{z}_k) \, g_r\right)\right).
\end{equation}
We can get the final updated core function as follows:
\begin{equation}
    \mathbf{s}_k = f_k \cdot \phi\!\big(\mathbf{A}\mathbf{s}_{k-1} + \mathbf{B}\mathbf{x}_k\big) + g_k \cdot \mathbf{d}_k\label{eq:s_k_full}.
\end{equation}

\begin{algorithm}[!t]
\caption{Fuzzy Semantic System}
\label{alg:fuzzy_system_compact}
\begin{algorithmic}[1]
\REQUIRE 
  Semantic context $\mathbf{z} \in \mathbb{R}^{B \times H}$ or $\mathbb{R}^{B \times T \times H}$, Parameters: rule centers $\{\mathbf{c}_r\}_{r=1}^{R} \subset \mathbb{R}^{H}$, 
  rule weights $\{f_r\}_{r=1}^{R}, \{g_r\}_{r=1}^{R} \subset \mathbb{R}$, width $\sigma > 0$
\ENSURE 
  Modulation vectors $\hat{\mathbf{f}}, \hat{\mathbf{g}}$ (same shape as $\mathbf{z}$)

\STATE $\mu_r(\mathbf{z}) = \dfrac{\exp\left(-\dfrac{\|\mathbf{z} - \mathbf{c}_r\|_2^2}{2\sigma^2}\right)}{\sum_{j=1}^{R} \exp\left(-\dfrac{\|\mathbf{z} - \mathbf{c}_j\|_2^2}{2\sigma^2}\right) + \varepsilon}$ \hfill $\triangleright$ Normalized membership

\STATE $\hat{f}_{\text{scalar}} = \sum_{r=1}^{R} \mu_r(\mathbf{z}) \cdot f_r$ \hfill $\triangleright$ Weighted average for $f$

\STATE $\hat{g}_{\text{scalar}} = \text{softplus}\!\left(\sum_{r=1}^{R} \mu_r(\mathbf{z}) \cdot g_r\right)$ \hfill $\triangleright$ Weighted average for $g$ (non-negative)

\STATE $\hat{\mathbf{f}} = \hat{f}_{\text{scalar}} \cdot \mathbf{1}_H,\quad \hat{\mathbf{g}} = \hat{g}_{\text{scalar}} \cdot \mathbf{1}_H$ \hfill $\triangleright$ Expand to hidden dimension

\RETURN $\hat{\mathbf{f}}, \hat{\mathbf{g}}$
\end{algorithmic}
\end{algorithm}

\subsection{Time Delay System}
In Eq.~\eqref{eq:s_k_full}, we need to recursively get $s_k$ using the last state $s_{k-1}$. To achieve fully vectorized computation without sequential loops, we restrict the state transition matrix $\mathbf{A}$ to be diagonal. Let $\mathbf{a} = \operatorname{diag}(\mathbf{A}) \in \mathbb{R}^d$ denote the diagonal elements. For $\mathbf{A}\mathbf{s}_{k-1} + \mathbf{B}\mathbf{x}_k$, we apply the linear recurrence with diagonal state transition. The latent state $\mathbf{h}_k \in \mathbb{R}^d$ evolves according to the element-wise recurrence:
\begin{equation}
    \mathbf{h}_k = \mathbf{a} \odot \mathbf{h}_{k-1} + \mathbf{b}_k, \quad \mathbf{h}_0 = \mathbf{0},
    \label{eq:diag_recurrence}
\end{equation}

where $\odot$ denotes element-wise (Hadamard) multiplication, and $\mathbf{b}_k = \mathbf{B} \mathbf{x}_k \in \mathbb{R}^d$ is the projected input.

\textbf{Closed-Form Solution via Cumulative Sum}. Unrolling the recurrence yields the closed-form expression:

\begin{equation}
    \mathbf{h}_k = \sum_{i=1}^{k} \left( \mathbf{a}^{\odot (k-i)} \odot \mathbf{b}_i \right),
    \label{eq:closed_form}
\end{equation}

where $\mathbf{a}^{\odot p}$ denotes element-wise exponentiation: $(\mathbf{a}^{\odot p})_j = (a_j)^p$. This can be computed efficiently using $\operatorname{cumsum}$ operation:

\begin{equation}
    \mathbf{h}_k = \mathbf{a}^{\odot k} \odot \operatorname{cumsum}_{i=1}^{k} \left( \mathbf{a}^{\odot -i} \odot \mathbf{b}_i \right),
    \label{eq:cumsum_form}
\end{equation}

where $\mathbf{a}^{\odot -i}$ denotes element-wise reciprocal: $(\mathbf{a}^{\odot -i})_j = (a_j)^{-i}$.

\textbf{Final State Update}. The semantic state $\mathbf{s}_k$ is obtained by combining the transformed latent state and delay correction with fuzzy modulation:

\begin{equation}
    \mathbf{s}_k = \hat{\mathbf{f}}_k \odot \phi(\mathbf{h}_k) + \hat{\mathbf{g}}_k \odot \mathbf{d}_k,
    \label{eq:final_state}
\end{equation}

where $\phi$ is a nonlinear activation function (e.g., SiLU).


\begin{algorithm}[!htb]
\caption{Semantic Delay Processor (Full version is shown in Alg.~\ref{alg:semantic_delay_processor}.)}
\label{alg:semantic_delay_processor_compact}
\begin{algorithmic}[1]
\REQUIRE Input tokens $I \in \mathbb{N}^{B \times T}$, parameters $\Theta$
\ENSURE Final state $\mathbf{s}_{\text{final}} \in \mathbb{R}^{B \times H}$, average state $\bar{\mathbf{s}} \in \mathbb{R}^{B \times H}$

\STATE $\mathbf{X} = \text{Dropout}(\text{Embed}(I) + \text{PosEnc}(1:T))$ \hfill $\triangleright$ Embedding
\STATE $\mathbf{X}_{\text{conv}} = \text{GELU}(\text{Conv}_{\text{bi}}(\mathbf{X}))$ \hfill $\triangleright$ Local context

\STATE $(\boldsymbol{\alpha},\ \bar{\mathbf{X}}^{\text{hist}},\ \hat{s}) = \text{LagEstimator}(\mathbf{X},\ \text{History}(\mathbf{X}))$ \hfill $\triangleright$ Algorithm 1

\STATE $\mathbf{b} = \mathbf{B}\mathbf{X} \in \mathbb{R}^{B \times T \times H}$ \hfill $\triangleright$ Projected input
\STATE $\mathbf{d} = \gamma \tanh(\hat{s} \odot \mathbf{C}\bar{\mathbf{X}}^{\text{hist}})$ \hfill $\triangleright$ Delay correction

\STATE $\mathbf{a} = \text{diag}(\mathbf{A})$, $\mathbf{a}^{\text{pow}}_t = \mathbf{a}^{\odot t}$
\STATE $\mathbf{h} = \mathbf{a}^{\odot t} \odot \text{CumSum}\!\left(\mathbf{b} \odot \mathbf{a}^{\odot -t}\right)$ \hfill $\triangleright$ Parallel scan

\STATE $(\hat{\mathbf{f}},\ \hat{\mathbf{g}}) = \text{FuzzySystem}(\mathbf{d}/2)$ \hfill $\triangleright$ Algorithm 2

\STATE $\mathbf{s} = \hat{\mathbf{f}} \odot \text{SiLU}(\mathbf{h}) + \hat{\mathbf{g}} \odot \mathbf{d}$ \hfill $\triangleright$ State update
\STATE $\mathbf{s} = \text{LayerNorm}(\mathbf{W}_{\text{out}}\mathbf{s} + \mathbf{X})$ \hfill $\triangleright$ Residual + output

\RETURN $\mathbf{s}[:,-1,:],\ \text{Mean}(\mathbf{s}, \text{dim}=1)$
\end{algorithmic}
\end{algorithm}

\section{Experiments}
\subsection{Datasets and implementations}
\label{sec:data}
In this paper, we test our TDS model on the \textbf{IMDB}~\citep{maas2011learning} dataset for text classification and the \textbf{CNN / DailyMail}~\citep{nallapati2016abstractive} dataset for text generation. IMDB has 25,000 training and 25,000 test samples, while CNN / DailyMail contains over 300,000 unique news samples. We use the full dataset for text classification training and testing, while for the generation task, we use 80,000 samples for training and 2,000 each for validation and testing. The detailed algorithms for text classification and generation are shown in Alg.~\ref{alg:text_classification} and Alg.~\ref{alg:text_generation} in the Appendix.

We run our experiments using PyTorch on an Nvidia H100 80G GPU and use the Adam optimizer. We set the learning rate to 3e-4 for text classification and 5e-5 for text generation; the batch size is 32; the number of iterations is 50 for classification and 100 for generation. We set the dropout rate to 0.3 to improve model generalization and avoid overfitting during training. Our proposed TDS model employs a hidden dimension of 256, a maximum semantic delay length of 30, and 6 fuzzy inference rules for adaptive state updating. \footnote{Source code is avaliable at \url{https://github.com/YoushanZhang/TDS}.}.

\subsection{Results}
\label{sec:results}

For the sequence modeling experiments, we also compare the performance of our TDS model with other state-of-the-art models.
As shown in Tables~\ref{tab:classification} and \ref{tab:generation}, our proposed TDS model outperforms state-of-the-art methods on text classification and generation tasks, achieving higher accuracy and ROUGE scores while maintaining a relatively small parameter count and efficient training time. However, due to the limited number of benchmark experiments, the model's broader generalizability remains untested. Meanwhile, the TDS model is also more sensitive to hyperparameter tuning.

\begin{table}[htbp]
    \centering
    \caption{Number of parameters and accuracy of the IMDB dataset.}
    \label{tab:classification}
    \begin{tabular}{lccc}
        \hline
        Model & Parameters(M)  & Accuracy & Total Training Time(s) \\
        \hline
        LSTM~\citep{qaisar2020sentiment}& 2.8 & 89.90 & 2300 \\
        Mamba~\citep{gu2023mamba}& 3.2  & 83.00 & 2300 \\
        RoBERTa-base~\citep{abidin2025comprehensive}& 125  & 94.13 & - \\
        DistilBERT-base~\citep{abidin2025comprehensive}& 66  & 92.13 & - \\
        \hline
        TDS & \textbf{2.6} & \textbf{95.35} & \textbf{2100} \\
        \hline
    \end{tabular}
\end{table}

\begin{table}[htbp]
    \centering
    \caption{Number of parameters and ROUGE score of the CNN/DailyMail dataset.}
    \label{tab:generation}
    \resizebox{\linewidth}{!}{%
    \begin{tabular}{lccc}
        \hline
        Model & Parameters(M)  & ROUGE-L & Total Training Time(s) \\
        \hline
        LSTM \citep{abdulsahib2022convolutional}& 26.5 &  38.45 & 16558 \\
        BERT \citep{maharditha2025comparative}& 24.30  & 40.48 & 18260 \\
        Mamba ~\citep{do2025discrete}& 14.73  &  29.60 & 16725 \\
        BART CNN-large ~\citep{rani2025deep}& 139  &  24.58 & - \\
        T5-base ~\citep{rani2025deep}& 220  &  23.54 & - \\
        \hline
        TDS & 16.98 & 42.63& 17633 \\
        \hline
    \end{tabular}
    }
    \vspace{-0.3cm}
\end{table}

\begin{figure}[H]
    \centering
    \includegraphics[width=1\linewidth]{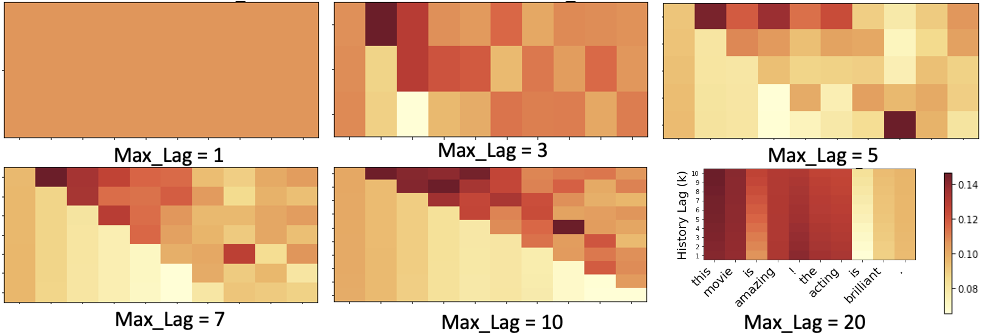}
    \vspace{-0.6cm}
    \caption{An example of how the maximum number of lag $L$ impacts TDS model performance on the sample ``this movie is amazing! the acting is brilliant.".}
    \label{fig:attention}
    \vspace{-0.3cm}
\end{figure}

We also visualize how different numbers of the maximum number of lags affect model performance, as shown in Fig.~\ref{fig:attention}, and Fig.~\ref{fig:lag} in the Appendix. The X-axis is a sample sentence of ten words: ``this movie is amazing ! the acting is brilliant.  The y-axis is the maximum number of historical lags. Given an example when Max\_Lag = 10, the word ``amazing" (position 3) attends to a maximum of the past 10 historical words. The figure shows that ``movie" has the biggest impact on the word ``amazing". We can see from Fig.~\ref{fig:attention} that as the number of lags increases, the importance of specific historical words to the current sequence becomes more prominent. While attention values can vary during training with different lags $L$, the most relevant historical context will be consistently highlighted by the attention mechanism. Fig.~\ref{fig:lag} shows that increasing lags, although with fluctuation, will increase the model accuracy on text classification. 

\section{Conclusion}
In this paper,  we introduced a novel sequence modeling framework based on time-delay systems (TDS), where a differentiable semantic lag estimator dynamically retrieves relevant history via content-aware attention. Our Semantic Delay Processor (SDP) explicitly extracts memories and follows an interpretable context-dependent processing rule. 
Experiments demonstrate that TDS outperforms RNN, Transformer, and Mamba on sentiment classification and text summarization tasks, while scaling linearly to million-length sequences. The learned fuzzy rule centers correspond to semantic contexts, offering interpretability. Our work explores time-delay systems as an efficient and interpretable mechanism for sequence modeling, which has proven effective for text tasks. It has the potential to be applied to solving more real-life problems, such as image and audio processing and generation, given its efficiency and offering greater interpretability.








\subsection*{AI use statement}
In preparing this version of the manuscript, we used generative AI tools to assist with conversion to the ICLR 2027 LaTeX format, package and file-path consistency checks, and the identification of compilation issues. We did not use generative AI tools in this formatting process to generate experimental data, conduct statistical analyses, select or fabricate references, or draw scientific conclusions. We reviewed all AI-assisted changes and verified the resulting LaTeX source and compiled document. We take responsibility for the final content of this work, including all text, claims, results, and artifacts.

\bibliography{ref}

@String{Computer = "{IEEE} Computer" }

@String{Springer = "Springer-Verlag" }

@article{sherstinsky2020fundamentals,
  title={Fundamentals of recurrent neural network (RNN) and long short-term memory (LSTM) network},
  author={Sherstinsky, Alex},
  journal={Physica d: Nonlinear phenomena},
  volume={404},
  pages={132306},
  year={2020},
  publisher={Elsevier}
}

@article{hochreiter1997long,
  title={Long short-term memory},
  author={Hochreiter, Sepp and Schmidhuber, J{\"u}rgen},
  journal={Neural computation},
  volume={9},
  number={8},
  pages={1735--1780},
  year={1997},
  publisher={MIT press}
}

@article{bengio1994learning,
  title={Learning long-term dependencies with gradient descent is difficult},
  author={Bengio, Yoshua and Simard, Patrice and Frasconi, Paolo},
  journal={IEEE transactions on neural networks},
  volume={5},
  number={2},
  pages={157--166},
  year={1994},
  publisher={IEEE}
}

@inproceedings{pascanu2013difficulty,
  title={On the difficulty of training recurrent neural networks},
  author={Pascanu, Razvan and Mikolov, Tomas and Bengio, Yoshua},
  booktitle={International conference on machine learning},
  pages={1310--1318},
  year={2013},
  organization={Pmlr}
}

@inproceedings{zhou2023deep,
  title={Deep latent state space models for time-series generation},
  author={Zhou, Linqi and Poli, Michael and Xu, Winnie and Massaroli, Stefano and Ermon, Stefano},
  booktitle={International Conference on Machine Learning},
  pages={42625--42643},
  year={2023},
  organization={PMLR}
}

@article{smith2208simplified,
  title={Simplified state space layers for sequence modeling, 2023},
  author={Smith, Jimmy TH and Warrington, Andrew and Linderman, Scott W},
  journal={URL https://arxiv. org/abs/2208.04933},
  volume={2208},
  pages={04933},
  year={2023}
}

@article{gu2023mamba,
  title={Mamba: Linear-time sequence modeling with selective state spaces},
  author={Gu, Albert and Dao, Tri},
  journal={arXiv preprint arXiv:2312.00752},
  year={2023}
}

@misc{kalman1960new,
  title={A new approach to linear filtering and prediction problems},
  author={Kalman, Rudolph Emil},
  year={1960}
}

@book{kailath1980linear,
  title={Linear systems},
  author={Kailath, Thomas},
  volume={156},
  year={1980},
  publisher={Prentice-Hall Englewood Cliffs, NJ}
}

@article{gu2021efficiently,
  title={Efficiently modeling long sequences with structured state spaces},
  author={Gu, Albert and Goel, Karan and R{\'e}, Christopher},
  journal={arXiv preprint arXiv:2111.00396},
  year={2021}
}

@article{rangapuram2018deep,
  title={Deep state space models for time series forecasting},
  author={Rangapuram, Syama Sundar and Seeger, Matthias W and Gasthaus, Jan and Stella, Lorenzo and Wang, Yuyang and Januschowski, Tim},
  journal={Advances in neural information processing systems},
  volume={31},
  year={2018}
}

@book{hale2013introduction,
  title={Introduction to functional differential equations},
  author={Hale, Jack K and Lunel, Sjoerd M Verduyn},
  volume={99},
  year={2013},
  publisher={Springer Science \& Business Media}
}

@book{insperger2011semi,
  title={Semi-discretization for time-delay systems: stability and engineering applications},
  author={Insperger, Tam{\'a}s and St{\'e}p{\'a}n, G{\'a}bor},
  volume={178},
  year={2011},
  publisher={Springer Science \& Business Media}
}

@book{smith2011introduction,
  title={An introduction to delay differential equations with applications to the life sciences},
  author={Smith, Hal L},
  volume={57},
  year={2011},
  publisher={springer New York}
}

@incollection{waibel2013phoneme,
  title={Phoneme recognition using time-delay neural networks},
  author={Waibel, Alexander and Hanazawa, Toshiyuki and Hinton, Geoffrey and Shikano, Kiyohiro and Lang, Kevin J},
  booktitle={Backpropagation},
  pages={35--61},
  year={2013},
  publisher={Psychology Press}
}

@inproceedings{snyder2015time,
  title={Time delay deep neural network-based universal background models for speaker recognition},
  author={Snyder, David and Garcia-Romero, Daniel and Povey, Daniel},
  booktitle={2015 IEEE Workshop on Automatic Speech Recognition and Understanding (ASRU)},
  pages={92--97},
  year={2015},
  organization={IEEE}
}

@article{huang2019deeper,
  title={Deeper time delay neural networks for effective acoustic modelling},
  author={Huang, Xiaorong and Zhang, Weibin and Xu, Xiangmin and Yin, Ruixiang and Chen, Dongpeng},
  journal={Journal of Physics: Conference Series},
  volume={1229},
  number={1},
  pages={012076},
  year={2019},
  organization={IOP Publishing}
}

@inproceedings{peddinti2015time,
  title={A time delay neural network architecture for efficient modeling of long temporal contexts},
  author={Peddinti, Vijayaditya and Povey, Daniel and Khudanpur, Sanjeev},
  booktitle={Proc. Interspeech 2015},
  pages={3214--3218},
  year={2015}
}

@article{ji2024trainable,
  title={Trainable delays in time delay neural networks for learning delayed dynamics},
  author={Ji, Xunbi A and Orosz, G{\'a}bor},
  journal={IEEE Transactions on Neural Networks and Learning Systems},
  volume={36},
  number={3},
  pages={5219--5229},
  year={2024},
  publisher={IEEE}
}

@article{jatnika2019word2vec,
  title={Word2vec model analysis for semantic similarities in english words},
  author={Jatnika, Derry and Bijaksana, Moch Arif and Suryani, Arie Ardiyanti},
  journal={Procedia Computer Science},
  volume={157},
  pages={160--167},
  year={2019},
  publisher={Elsevier}
}

@inproceedings{sivakumar2020review,
  title={Review on word2vec word embedding neural net},
  author={Sivakumar, Soubraylu and Videla, Lakshmi Sarvani and Kumar, T Rajesh and Nagaraj, J and Itnal, Shilpa and Haritha, D},
  booktitle={2020 international conference on smart electronics and communication (ICOSEC)},
  pages={282--290},
  year={2020},
  organization={IEEE}
}

@inproceedings{maas2011learning,
  title={Learning word vectors for sentiment analysis},
  author={Maas, Andrew and Daly, Raymond E and Pham, Peter T and Huang, Dan and Ng, Andrew Y and Potts, Christopher},
  booktitle={Proceedings of the 49th annual meeting of the association for computational linguistics: Human language technologies},
  pages={142--150},
  year={2011}
}

@inproceedings{nallapati2016abstractive,
  title={Abstractive text summarization using sequence-to-sequence rnns and beyond},
  author={Nallapati, Ramesh and Zhou, Bowen and Dos Santos, Cicero and Gul{\c{c}}ehre, {\c{C}}a{\u{g}}lar and Xiang, Bing},
  booktitle={Proceedings of the 20th SIGNLL conference on computational natural language learning},
  pages={280--290},
  year={2016}
}

@article{abdulsahib2022convolutional,
  title={Convolutional Recurrent Neural Networks for Text Lecture Summarization},
  author={Abdulsahib, Muna Ghazi and others},
  journal={Iraqi Journal of Computers, Communications, Control and Systems Engineering},
  volume={22},
  number={2},
  pages={3},
  year={2022}
}

@inproceedings{maharditha2025comparative,
  title={Comparative Performance Analysis of LSTM and Transformer Models for Abstractive Text Summarization},
  author={Maharditha, I Made Denis and Lhaksmana, Kemas M},
  booktitle={2025 International Conference on Data Science and Its Applications (ICoDSA)},
  pages={213--218},
  year={2025},
  organization={IEEE}
}

@inproceedings{do2025discrete,
  title={Discrete diffusion language model for efficient text summarization},
  author={Do, Duc Anh and Tuan, Luu Anh and Buntine, Wray and others},
  booktitle={Findings of the Association for Computational Linguistics: NAACL 2025},
  pages={6278--6290},
  year={2025}
}

@inproceedings{qaisar2020sentiment,
  title={Sentiment analysis of IMDb movie reviews using long short-term memory},
  author={Qaisar, Saeed Mian},
  booktitle={2020 2nd International Conference on Computer and Information Sciences (ICCIS)},
  pages={1--4},
  year={2020},
  organization={IEEE}
}

@article{abidin2025comprehensive,
  title={A comprehensive benchmarking pipeline for transformer-based sentiment analysis using cross-validated metrics},
  author={Abidin, Dodo Zaenal and Afuan, Lasmedi and Toscany, Afrizal Nehemia and Nurhadi, Nurhadi},
  journal={Jurnal Teknik Informatika (Jutif)},
  volume={6},
  number={4},
  pages={1797--1810},
  year={2025}
}

@article{rani2025deep,
  title={Deep learning for text summarization using NLP for automated news digest},
  author={Rani Krishna, KM and Somasundaram, K and Arulmozhivarman, P and Immanuel, Sarah A and Rajkumar, ER},
  journal={Scientific Reports},
  volume={15},
  number={1},
  pages={36343},
  year={2025},
  publisher={Nature Publishing Group UK London}
}
\bibliographystyle{iclr2027_conference}

\newpage
\appendix
\renewcommand{\theHfigure}{appendix.\Alph{section}.\arabic{figure}}
\renewcommand{\theHtable}{appendix.\Alph{section}.\arabic{table}}
\section{Appendix}

\subsection{Application algorithm}
The detailed algorithms for text classification and text generation are shown in Alg.~\ref{alg:text_classification} and Alg.~\ref{alg:text_generation}. 

\begin{algorithm}[!htb]
\caption{Text Classification with Semantic Delay Processor}
\label{alg:text_classification}
\begin{algorithmic}[1]
\REQUIRE Input $I \in \mathbb{N}^{B \times T}$, classes $C$
\ENSURE Probabilities $\mathbf{p} \in \mathbb{R}^{B \times C}$, predictions $\hat{y} \in \mathbb{N}^{B}$

\STATE $\mathbf{X}_{\text{pool}} = \text{Mean}(\text{Emb}(I) + \text{PosEnc})$ \hfill $\triangleright$ $[B,H]$
\STATE $(\mathbf{s}_{\text{final}},\ \bar{\mathbf{s}}) = \text{Processor}(I)$ \hfill $\triangleright$ Alg. \ref{alg:semantic_delay_processor_compact}
\STATE $\mathbf{s} = \text{LayerNorm}(\mathbf{W}_{\text{fuse}}[\mathbf{s}_{\text{final}};\bar{\mathbf{s}}] + \mathbf{W}_{\text{short}}\mathbf{X}_{\text{pool}})$
\STATE $\mathbf{p} = \text{Softmax}(\text{MLP}(\mathbf{s}))$ \hfill $\triangleright$ $\text{MLP}: H \to H/2 \to C$
\RETURN $\mathbf{p},\ \arg\max(\mathbf{p})$
\end{algorithmic}
\end{algorithm}

\begin{algorithm}[!htb]
\caption{Text Generation with Semantic Delay Processor}
\label{alg:text_generation}
\begin{algorithmic}[1]
\REQUIRE Prefix $I_{\text{p}} \in \mathbb{N}^{B \times T_{\text{p}}}$, length $T_{\text{gen}}$, temp $\tau$, top-$k$
\ENSURE Generated sequences $I_{\text{gen}} \in \mathbb{N}^{B \times (T_{\text{p}} + T_{\text{gen}})}$

\STATE $I_{\text{gen}} \leftarrow I_{\text{p}},\quad \mathbf{s} \leftarrow \text{Processor}(I_{\text{p}})$

\FOR{$t = 1$ \textbf{to} $T_{\text{gen}}$}
  \STATE $\mathbf{x} = \text{Emb}(I_{\text{gen}}[:,-1]) + \text{PosEnc}(T_{\text{p}}+t)$
  \STATE $\mathbf{b} = \mathbf{B}\mathbf{x},\quad \mathbf{d} = \gamma \tanh(\hat{s} \odot \mathbf{C}\bar{\mathbf{X}}^{\text{hist}})$
  \STATE $\mathbf{h} = \mathbf{a}^{\odot t} \odot \text{CumSum}(\mathbf{b} \odot \mathbf{a}^{\odot -t})$
  \STATE $(\mathbf{f},\mathbf{g}) = \text{Fuzzy}(\mathbf{d}/2),\quad \mathbf{s} = \mathbf{f} \odot \text{SiLU}(\mathbf{h}) + \mathbf{g} \odot \mathbf{d}$
  \STATE $\mathbf{z} = \mathbf{W}_{\text{out}}\mathbf{s},\quad \mathbf{z}[\neg\text{TopK}(\mathbf{z}/\tau,k)] = -\infty$
  \STATE $w \sim \text{Categorical}(\text{Softmax}(\mathbf{z}))$
  \STATE $I_{\text{gen}} \leftarrow \text{Concat}(I_{\text{gen}}, w)$
\ENDFOR

\RETURN $I_{\text{gen}}$
\end{algorithmic}
\end{algorithm}

\begin{algorithm}[!htb]
\caption{Semantic Delay Processor (Fully Version)}
\label{alg:semantic_delay_processor}
\begin{algorithmic}[1]
\REQUIRE 
  Input tokens $I \in \mathbb{N}^{B \times T}$ (batch size $B$, sequence length $T$), \\
  Vocabulary size $V$, hidden dimension $H$, maximum lag $L$, number of fuzzy rules $R$, \\
  Hyperparameters: $\kappa, \lambda, \alpha, \sigma, \gamma$, kernel size $K$
\ENSURE 
  Final semantic state $\mathbf{s}_{\text{final}} \in \mathbb{R}^{B \times H}$, \\
  Average semantic state $\bar{\mathbf{s}} \in \mathbb{R}^{B \times H}$

\STATE \textbf{Step 1: Token Embedding and Position Encoding}
\STATE $\mathbf{X} = \text{Embedding}(I) \in \mathbb{R}^{B \times T \times H}$ \hfill $\triangleright$ Lookup embedding table
\STATE $\mathbf{X} = \mathbf{X} + \text{PosEnc}(1:T) \in \mathbb{R}^{B \times T \times H}$ \hfill $\triangleright$ Add bidirectional position encoding
\STATE $\mathbf{X} = \text{Dropout}(\mathbf{X})$

\STATE \textbf{Step 2: Bidirectional Depthwise Convolution}
\STATE $\mathbf{X}_{\text{proj}} = \mathbf{W}_{\text{in}} \mathbf{X} \in \mathbb{R}^{B \times T \times H}$ \hfill $\triangleright$ Input projection
\STATE $\mathbf{X}_{\text{conv}}^{\text{fwd}} = \text{GELU}(\text{Conv1D}_{\text{fwd}}(\mathbf{X}_{\text{proj}}))$ \hfill $\triangleright$ Forward convolution
\STATE $\mathbf{X}_{\text{conv}}^{\text{bwd}} = \text{GELU}(\text{Conv1D}_{\text{bwd}}(\text{Flip}(\mathbf{X}_{\text{proj}})))$ \hfill $\triangleright$ Backward convolution
\STATE $\mathbf{X}_{\text{conv}} = \mathbf{X}_{\text{conv}}^{\text{fwd}} + \mathbf{X}_{\text{conv}}^{\text{bwd}} \in \mathbb{R}^{B \times T \times H}$

\STATE \textbf{Step 3: History Window Construction}
\STATE $\mathbf{X}_{\text{padded}} = [\mathbf{0}_{B \times L \times H}, \mathbf{X}] \in \mathbb{R}^{B \times (T+L) \times H}$
\STATE $\mathbf{H}_{\text{hist}} = \text{Unfold}(\mathbf{X}_{\text{padded}}, \text{window}=L, \text{step}=1) \in \mathbb{R}^{B \times T \times L \times H}$

\STATE \textbf{Step 4: Soft Semantic Lag Estimation (Algorithm \ref{alg:soft_lag_estimator_compact})}
\STATE $(\boldsymbol{\alpha},\ \bar{\mathbf{X}}^{\text{hist}},\ \hat{s}^{\text{soft}}) = \text{LagEstimator}(\mathbf{X},\ \mathbf{H}_{\text{hist}})$
\STATE \hfill $\triangleright$ $\boldsymbol{\alpha} \in \mathbb{R}^{B \times T \times L}$, $\bar{\mathbf{X}}^{\text{hist}} \in \mathbb{R}^{B \times T \times H}$, $\hat{s}^{\text{soft}} \in \mathbb{R}^{B \times T}$

\STATE \textbf{Step 5: Compute Input Projection}
\STATE $\mathbf{b}_k = \mathbf{B} \cdot (\mathbf{X} + \mathbf{X}_{\text{conv}})^{\top} \in \mathbb{R}^{B \times T \times H}$ \hfill $\triangleright$ Projected input

\STATE \textbf{Step 6: Compute Delay Correction}
\STATE $\bar{\mathbf{X}}^{\text{proj}} = \mathbf{C} \cdot \bar{\mathbf{X}}^{\text{hist}} \in \mathbb{R}^{B \times T \times H}$ \hfill $\triangleright$ Project historical context
\STATE $\mathbf{d}_k = \gamma \cdot \tanh\!\left(\hat{s}^{\text{soft}} \odot \bar{\mathbf{X}}^{\text{proj}}\right) \in \mathbb{R}^{B \times T \times H}$ \hfill $\triangleright$ Delay correction term

\STATE \textbf{Step 7: Diagonal State Space Recurrence (Zero-Loop)}
\STATE $\mathbf{a} = \text{diag}(\mathbf{A}) \in \mathbb{R}^{H}$ \hfill $\triangleright$ Diagonal state transition
\STATE $t = [0, 1, \ldots, T-1] \in \mathbb{R}^{T}$ \hfill $\triangleright$ Time indices
\STATE $\mathbf{a}^{\text{pow}}[t, h] = (a_h)^{t} \in \mathbb{R}^{T \times H}$ \hfill $\triangleright$ Exponential powers
\STATE $\mathbf{b}_{\text{scaled}} = \mathbf{b}_k \odot (\mathbf{a}^{\text{pow}} + \varepsilon)^{-1} \in \mathbb{R}^{B \times T \times H}$
\STATE $\mathbf{h}_{\text{scaled}} = \text{CumSum}(\mathbf{b}_{\text{scaled}}, \text{dim}=1) \in \mathbb{R}^{B \times T \times H}$
\STATE $\mathbf{h}_k = \mathbf{h}_{\text{scaled}} \odot \mathbf{a}^{\text{pow}} \in \mathbb{R}^{B \times T \times H}$ \hfill $\triangleright$ Latent state from recurrence

\STATE \textbf{Step 8: Fuzzy Semantic Modulation (Algorithm \ref{alg:fuzzy_system_compact})}
\STATE $\mathbf{z} = \dfrac{\mathbf{d}_k}{2} \in \mathbb{R}^{B \times T \times H}$ \hfill $\triangleright$ Semantic context (approximation)
\STATE $(\hat{\mathbf{f}},\ \hat{\mathbf{g}}) = \text{FuzzySystem}(\mathbf{z}) \in \mathbb{R}^{B \times T \times H}$ \hfill $\triangleright$ Modulation vectors

\STATE \textbf{Step 9: Final State Update}
\STATE $\mathbf{s}_k = \hat{\mathbf{f}} \odot \text{SiLU}(\mathbf{h}_k) + \hat{\mathbf{g}} \odot \mathbf{d}_k \in \mathbb{R}^{B \times T \times H}$
\STATE $\mathbf{s}_k = \text{Dropout}(\mathbf{s}_k)$

\STATE \textbf{Step 10: Residual Connection and Output Projection}
\STATE $\mathbf{s}_k = \mathbf{W}_{\text{out}} \mathbf{s}_k \in \mathbb{R}^{B \times T \times H}$ \hfill $\triangleright$ Output projection
\STATE $\mathbf{s}_k = \text{LayerNorm}(\mathbf{s}_k + \mathbf{X})$ \hfill $\triangleright$ Residual connection

\STATE \textbf{Step 11: Extract Output States}
\STATE $\mathbf{s}_{\text{final}} = \mathbf{s}_k[:, -1, :] \in \mathbb{R}^{B \times H}$ \hfill $\triangleright$ Final token state
\STATE $\bar{\mathbf{s}} = \text{Mean}(\mathbf{s}_k, \text{dim}=1) \in \mathbb{R}^{B \times H}$ \hfill $\triangleright$ Average state

\RETURN $\mathbf{s}_{\text{final}},\ \bar{\mathbf{s}}$
\end{algorithmic}
\end{algorithm}

\subsection{Computational Complexity}

The full model operations are defined in Eq.\eqref{eq:final_state}, which are fully vectorized:
\begin{itemize}
    \item $\mathbf{a}^{\odot k}$ and $\mathbf{a}^{\odot -i}$ are computed via broadcasted exponentiation
    \item $\operatorname{cumsum}$ is a single GPU-accelerated primitive
    \item All other operations are element-wise or matrix multiplications
\end{itemize}

The overall time complexity is $\mathcal{O}(1)$ with respect to sequence length when executed on parallel hardware, as the $\operatorname{cumsum}$ operation has $\mathcal{O}(\log T)$ in depth but implemented as a single kernel call.

\subsection{Special Cases}

\begin{align}
\text{If } \mathbf{a} = \mathbf{1}: &\quad \mathbf{h}_k = \operatorname{cumsum}_{i=1}^{k} \mathbf{b}_i \quad \text{(standard cumulative sum)} \\
\text{If } \mathbf{a} \to \mathbf{0}: &\quad \mathbf{h}_k \approx \mathbf{b}_k \quad \text{(no memory)} \\
\text{If } \gamma = 0: &\quad \mathbf{s}_k = f_k \cdot \phi(\mathbf{h}_k) \quad \text{(delay correction disabled)}
\end{align}

\subsection{Effects of different maximum lag and accuracy}

\begin{figure}
    \centering
    \includegraphics[width=0.9\linewidth]{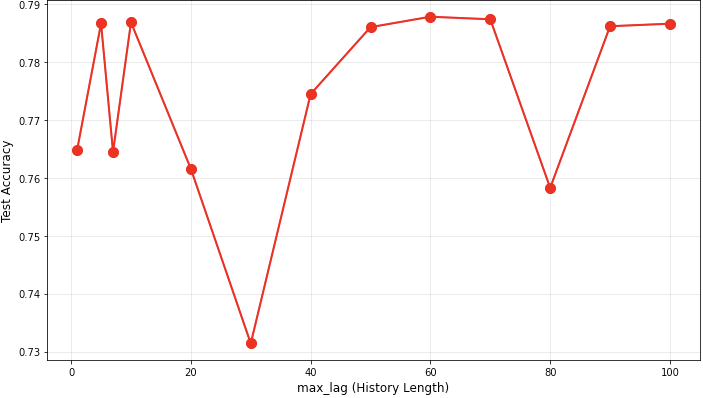}
    \caption{Effects of different lag on IMDB dataset}
    \label{fig:lag}
\end{figure}

\section{Theoretical Properties}
\label{sec:theory}

The nonlinear time-delay system defined in Eq.~\eqref{eq:nonlinear_delay_dde} and its discrete realization
in Eq.~\eqref{eq:discrete_state_update} possess several fundamental theoretical properties that guarantee the well-posed and stable dynamics. We summarize the key results as follows.

\subsection{Existence and Uniqueness of Solutions}

Consider the nonlinear time-delay system
\begin{equation}
\dot{\mathbf{s}}(t)
= \phi\!\big(\mathbf{A}\mathbf{s}(t) + \mathbf{B}\mathbf{x}(t)\big)
+ \gamma\,\psi\!\big(\hat{s}(t)\,\mathbf{C}\,\bar{\mathbf{x}}(t-\tau(t))\big),
\label{eq:delay_system}
\end{equation}
where $\mathbf{s}(t) \in \mathbb{R}^d$ is the state, $\mathbf{x}(t)$ is the input,
$\mathbf{A}, \mathbf{B}, \mathbf{C}$ are constant matrices,
$\phi, \psi$ are nonlinear activation functions,
$\gamma > 0$ is a gain, $\hat{s}(t) \in [0,1]$ is the similarity score,
and $\tau(t)$ is a time-varying delay satisfying $0 \leq \tau(t) \leq \tau_{\max}$.

\begin{theorem}[Existence and Uniqueness]
\label{thm:existence_uniqueness}
Assume that the nonlinear functions $\phi(\cdot)$ and $\psi(\cdot)$ are Lipschitz continuous
with Lipschitz constants $L_\phi$ and $L_\psi$, respectively.
Then, for any bounded initial history function 
$\mathbf{s}(t) = \mathbf{s}_0(t)$ defined on $t \in [-\tau_{\max}, 0]$,
the nonlinear time-delay system \eqref{eq:delay_system}
admits a unique continuous solution $\mathbf{s}(t)$ on any finite interval $t \in [0,T]$.
\end{theorem}

\begin{proof}
The proof reformulates the delay differential equation as an integral equation and applies the Banach fixed-point theorem.

Let $\mathcal{C} = C([-\tau_{\max}, T], \mathbb{R}^d)$ be the Banach space of continuous functions equipped with the supremum norm $\|\mathbf{s}\|_{\infty} = \sup_{t} \|\mathbf{s}(t)\|$. Define the closed subset
\[
\mathcal{B} = \left\{ \mathbf{s} \in \mathcal{C} : \mathbf{s}(t) = \mathbf{s}_0(t) \text{ for } t \in [-\tau_{\max}, 0], \; \|\mathbf{s} - \mathbf{s}_0\|_{\infty} \leq R \right\}
\]
for some $R > 0$ to be determined.

For $t \in [0,T]$, any solution to \eqref{eq:delay_system} must satisfy the integral equation
\[
\mathbf{s}(t) = \mathbf{s}(0) + \int_0^t \Big[ \phi\!\big(\mathbf{A}\mathbf{s}(u) + \mathbf{B}\mathbf{x}(u)\big) + \gamma\,\psi\!\big(\hat{s}(u)\,\mathbf{C}\,\bar{\mathbf{x}}(u-\tau(u))\big) \Big] du.
\]

Define the operator $\mathcal{F}: \mathcal{B} \to \mathcal{C}$ by
\[
(\mathcal{F}\mathbf{s})(t) = 
\begin{cases}
\mathbf{s}_0(t), & t \in [-\tau_{\max}, 0], \\[4pt]
\displaystyle \mathbf{s}(0) + \int_0^t \Big[ \phi\!\big(\mathbf{A}\mathbf{s}(u) + \mathbf{B}\mathbf{x}(u)\big) + \gamma\,\psi\!\big(\hat{s}(u)\,\mathbf{C}\,\bar{\mathbf{x}}(u-\tau(u))\big) \Big] du, & t \in [0,T].
\end{cases}
\]

For any $\mathbf{s}_1, \mathbf{s}_2 \in \mathcal{B}$ and $t \in [0,T]$, the Lipschitz continuity of $\phi$ and $\psi$ yields
\[
\begin{aligned}
\|(\mathcal{F}\mathbf{s}_1)(t) - (\mathcal{F}\mathbf{s}_2)(t)\|
&\leq \int_0^t L_\phi \|\mathbf{A}\| \|\mathbf{s}_1(u) - \mathbf{s}_2(u)\| du \\
&\quad + \gamma \int_0^t L_\psi \|\mathbf{C}\| \|\bar{\mathbf{x}}\|_{\infty} |\hat{s}_1(u) - \hat{s}_2(u)| du,
\end{aligned}
\]
where $\|\bar{\mathbf{x}}\|_{\infty} = \sup_t \|\bar{\mathbf{x}}(t)\| < \infty$.

Since $\hat{s}(u)$ is a continuous function of $\mathbf{s}(u)$, there exists a constant $L_{\hat{s}} > 0$ such that $|\hat{s}_1(u) - \hat{s}_2(u)| \leq L_{\hat{s}} \|\mathbf{s}_1(u) - \mathbf{s}_2(u)\|$. Thus,
\[
\|(\mathcal{F}\mathbf{s}_1)(t) - (\mathcal{F}\mathbf{s}_2)(t)\| \leq K \int_0^t \|\mathbf{s}_1(u) - \mathbf{s}_2(u)\| du,
\]
where $K = L_\phi \|\mathbf{A}\| + \gamma L_\psi \|\mathbf{C}\| \|\bar{\mathbf{x}}\|_{\infty} L_{\hat{s}}$.

Taking the supremum over $t \in [0,T]$ gives
\[
\|\mathcal{F}\mathbf{s}_1 - \mathcal{F}\mathbf{s}_2\|_{\infty} \leq K T \|\mathbf{s}_1 - \mathbf{s}_2\|_{\infty}.
\]

Choose $T$ sufficiently small so that $K T < 1$. Then $\mathcal{F}$ is a contraction on $\mathcal{B}$. By the Banach fixed-point theorem, there exists a unique $\mathbf{s}^* \in \mathcal{B}$ such that $\mathcal{F}\mathbf{s}^* = \mathbf{s}^*$, which is the unique solution on $[0,T]$.

To extend to the entire interval $[0, T_{\text{final}}]$, repeat the argument. Let $T_1 = T$ and define $\mathbf{s}(t)$ on $[0, T_1]$. Then take the solution at $t = T_1$ as the new initial history and extend to $[T_1, 2T]$. Finitely many such steps yield a unique solution on $[0, T_{\text{final}}]$.

The Lipschitz continuity of $\phi$ and $\psi$ holds for common activation functions like $\tanh$ and $\text{SiLU}$. The input $\mathbf{x}(t)$ is bounded by construction (e.g., word embeddings), and the initial history $\mathbf{s}_0(t)$ is given and bounded. Hence, all conditions are satisfied.
\end{proof}

\subsection{Boundedness of the State Trajectory}


\begin{theorem}[Boundedness of Solutions]
\label{thm:boundedness}
Assume the following conditions hold:
\begin{enumerate}
    \item The input $\mathbf{x}(t)$ is uniformly bounded: there exists $M_x > 0$ such that $\|\mathbf{x}(t)\| \leq M_x$ for all $t \geq 0$.
    \item The activation functions $\phi$ and $\psi$ are Lipschitz continuous with constants $L_\phi, L_\psi > 0$ and satisfy $\phi(\mathbf{0}) = \mathbf{0}$, $\psi(\mathbf{0}) = \mathbf{0}$.
    \item The similarity score satisfies $0 \leq \hat{s}(t) \leq 1$ for all $t$.
\end{enumerate}
Then the state trajectory $\mathbf{s}(t)$ satisfies the following bound for all $t \geq 0$:
\begin{equation}
\|\mathbf{s}(t)\| \leq 
\|\mathbf{s}(0)\| e^{L_\phi \|\mathbf{A}\| t}
+ \frac{L_\phi \|\mathbf{B}\| M_x + \gamma L_\psi \|\mathbf{C}\| M_x}{L_\phi \|\mathbf{A}\|}
\left( e^{L_\phi \|\mathbf{A}\| t} - 1 \right),
\label{eq:boundedness}
\end{equation}
provided that $\mathbf{A} \neq \mathbf{0}$. If $\mathbf{A} = \mathbf{0}$, the bound simplifies to
\begin{equation}
\|\mathbf{s}(t)\| \leq \|\mathbf{s}(0)\| + \left( L_\phi \|\mathbf{B}\| M_x + \gamma L_\psi \|\mathbf{C}\| M_x \right) t.
\label{eq:boundedness_zero}
\end{equation}
\end{theorem}

\begin{proof}
Take the Euclidean norm on both sides of \eqref{eq:nonlinear_delay_dde}. By the triangle inequality,
\[
\|\dot{\mathbf{s}}(t)\| \leq 
\Big\| \phi\!\big(\mathbf{A}\mathbf{s}(t) + \mathbf{B}\mathbf{x}(t)\big) \Big\|
+ \gamma \Big\| \psi\!\big(\hat{s}(t)\,\mathbf{C}\,\mathbf{x}(t-\tau(t))\big) \Big\|.
\]

Since $\phi$ is Lipschitz continuous with $\phi(\mathbf{0}) = \mathbf{0}$, we have
\[
\Big\| \phi\!\big(\mathbf{A}\mathbf{s}(t) + \mathbf{B}\mathbf{x}(t)\big) \Big\|
\leq L_\phi \|\mathbf{A}\mathbf{s}(t) + \mathbf{B}\mathbf{x}(t)\|
\leq L_\phi \|\mathbf{A}\| \|\mathbf{s}(t)\| + L_\phi \|\mathbf{B}\| \|\mathbf{x}(t)\|.
\]

Similarly, using the Lipschitz continuity of $\psi$, the boundedness of $\hat{s}(t) \in [0,1]$, and $\psi(\mathbf{0}) = \mathbf{0}$,
\[
\Big\| \psi\!\big(\hat{s}(t)\,\mathbf{C}\,\mathbf{x}(t-\tau(t))\big) \Big\|
\leq L_\psi \|\mathbf{C}\| \|\mathbf{x}(t-\tau(t))\|
\leq L_\psi \|\mathbf{C}\| M_x.
\]

Substituting these bounds and using $\|\mathbf{x}(t)\| \leq M_x$, we obtain the differential inequality
\[
\|\dot{\mathbf{s}}(t)\| \leq a \|\mathbf{s}(t)\| + b,
\]
where $a = L_\phi \|\mathbf{A}\|$ and $b = L_\phi \|\mathbf{B}\| M_x + \gamma L_\psi \|\mathbf{C}\| M_x$.

This inequality can be rewritten as
\[
\frac{d}{dt} \left( \|\mathbf{s}(t)\| e^{-at} \right) \leq b e^{-at}.
\]

Integrating from $0$ to $t$ gives
\[
\|\mathbf{s}(t)\| e^{-at} - \|\mathbf{s}(0)\| \leq b \int_0^t e^{-a\tau} d\tau = \frac{b}{a} \left(1 - e^{-at}\right).
\]

Multiplying by $e^{at}$ yields
\[
\|\mathbf{s}(t)\| \leq \|\mathbf{s}(0)\| e^{at} + \frac{b}{a} \left(e^{at} - 1\right).
\]

Substituting back $a = L_\phi \|\mathbf{A}\|$ and $b = L_\phi \|\mathbf{B}\| M_x + \gamma L_\psi \|\mathbf{C}\| M_x$ proves \eqref{eq:boundedness}. When $\mathbf{A} = \mathbf{0}$, the differential inequality reduces to $\|\dot{\mathbf{s}}(t)\| \leq b$, and direct integration gives \eqref{eq:boundedness_zero}.
\end{proof}

This result guarantees that under mild assumptions, the system is free from divergence or blow-up,
and its trajectories evolve within a bounded region of the latent state space.

\subsection{Stability of the Equilibrium}

We next examine the stability of the equilibrium point.
Let $\mathbf{s}^*$ denote a fixed point satisfying
\begin{equation}
\mathbf{0}
= \phi\!\big(\mathbf{A}\mathbf{s}^* + \mathbf{B}\mathbf{x}^*\big)
+ \gamma\,\psi\!\big(\hat{s}^*\,\mathbf{C}\,\bar{\mathbf{x}}^*\big),
\end{equation}
for some constant input $\mathbf{x}^*$ and delay $\tau^*$.

\begin{theorem}[Local Exponential Stability]
If the spectral radius $\rho(\mathbf{A})<1$
and the nonlinearities $\phi$ and $\psi$ satisfy
$\|\nabla \phi\|\leq L_\phi$ and $\|\nabla \psi\|\leq L_\psi$,
then the equilibrium $\mathbf{s}^*$ is locally exponentially stable;
that is, there exist constants $K>0$ and $\eta>0$ such that
\begin{equation}\label{eq:exponential_stability}
\|\mathbf{s}(t)-\mathbf{s}^*\|
\leq K e^{-\eta t}\|\mathbf{s}(0)-\mathbf{s}^*\|.
\end{equation}
\end{theorem}


\begin{proof}
Define the error vector $\mathbf{e}(t) = \mathbf{s}(t) - \mathbf{s}^*$. Assuming $\mathbf{x}(t) = \mathbf{x}^*$ is constant in a neighborhood of the equilibrium, we have $\mathbf{x}(t) = \mathbf{x}^*$ and $\mathbf{x}(t-\tau(t)) = \mathbf{x}^*$. Then $\hat{s}(t) = \hat{s}^* + \Delta \hat{s}(t)$, where $\Delta \hat{s}(t) = \frac{\partial \hat{s}}{\partial \mathbf{s}} \mathbf{e}(t) + \mathcal{O}(\|\mathbf{e}\|^2)$.

Expanding $\phi$ and $\psi$ around the equilibrium using Taylor series:
\begin{align}
\phi\!\big(\mathbf{A}\mathbf{s}(t) + \mathbf{B}\mathbf{x}^*\big) &= 
\phi\!\big(\mathbf{A}\mathbf{s}^* + \mathbf{B}\mathbf{x}^*\big) + 
\nabla \phi(\mathbf{A}\mathbf{s}^* + \mathbf{B}\mathbf{x}^*) \cdot \mathbf{A} \mathbf{e}(t) + \mathcal{O}(\|\mathbf{e}\|^2), \\
\psi\!\big(\hat{s}(t)\,\mathbf{C}\,\mathbf{x}^*\big) &= 
\psi\!\big(\hat{s}^*\,\mathbf{C}\,\mathbf{x}^*\big) + 
\nabla \psi(\hat{s}^*\,\mathbf{C}\,\mathbf{x}^*) \cdot \mathbf{C}\,\mathbf{x}^* \cdot \frac{\partial \hat{s}}{\partial \mathbf{s}} \mathbf{e}(t) + \mathcal{O}(\|\mathbf{e}\|^2).
\end{align}
\begin{equation}
\mathbf{s}^* = \phi\!\big(\mathbf{A}\mathbf{s}^* + \mathbf{B}\mathbf{x}^*\big) + \gamma\,\psi\!\big(\hat{s}^*\,\mathbf{C}\,\bar{\mathbf{x}}^*\big).
\label{eq:equilibrium}
\end{equation}
Substituting these expansions into \eqref{eq:nonlinear_delay_dde} and using the equilibrium condition \eqref{eq:equilibrium}, we obtain the linearized error dynamics:
\begin{equation}
\dot{\mathbf{e}}(t) = \mathbf{A}_\phi \mathbf{e}(t) + \gamma \mathbf{C}_\psi \mathbf{e}(t - \tau^*),
\label{eq:linearized}
\end{equation}
where
\[
\mathbf{A}_\phi = \nabla \phi(\mathbf{A}\mathbf{s}^* + \mathbf{B}\mathbf{x}^*) \cdot \mathbf{A}, \qquad
\mathbf{C}_\psi = \nabla \psi(\hat{s}^*\,\mathbf{C}\,\mathbf{x}^*) \cdot \mathbf{C}\,\mathbf{x}^* \cdot \frac{\partial \hat{s}}{\partial \mathbf{s}}.
\]

From the Lipschitz assumptions, we have the bounds
\[
\|\mathbf{A}_\phi\| \leq L_\phi \|\mathbf{A}\|, \qquad 
\|\mathbf{C}_\psi\| \leq L_\psi \|\mathbf{C}\| \|\mathbf{x}^*\| L_{\hat{s}}.
\]

The characteristic equation of the linear delay system \eqref{eq:linearized} is
\begin{equation}
\det\!\big(\lambda \mathbf{I} - \mathbf{A}_\phi - \gamma e^{-\lambda \tau^*} \mathbf{C}_\psi\big) = 0.
\label{eq:characteristic}
\end{equation}

For any root $\lambda$ of \eqref{eq:characteristic}, taking norms gives
\[
|\lambda| \leq \|\mathbf{A}_\phi\| + \gamma \|\mathbf{C}_\psi\| \leq L_\phi \|\mathbf{A}\| + \gamma L_\psi \|\mathbf{C}\| \|\mathbf{x}^*\| L_{\hat{s}}.
\]

Since $\rho(\mathbf{A}) < 1$, there exists a matrix norm $\|\cdot\|_*$ such that $\|\mathbf{A}\|_* < 1$. By continuity, for sufficiently small $\gamma$ or sufficiently small $\|\mathbf{C}_\psi\|$, we also have $\rho(\mathbf{A}_\phi) < 1$. The small-gain condition
\begin{equation}
\|\mathbf{A}_\phi\| + \gamma \|\mathbf{C}_\psi\| < 1
\label{eq:small_gain}
\end{equation}
ensures that all characteristic roots satisfy $\Re(\lambda) < -\eta$ for some $\eta > 0$.

To establish exponential stability rigorously, construct the Lyapunov-Krasovskii functional
\begin{equation}
V(\mathbf{e}_t) = \mathbf{e}(t)^\top \mathbf{P} \mathbf{e}(t) + \int_{t-\tau_{\max}}^{t} \int_{\theta}^{t} \|\mathbf{e}(s)\|^2 ds d\theta,
\label{eq:lyapunov}
\end{equation}
where $\mathbf{P} = \mathbf{P}^\top > 0$ is a positive definite matrix. Taking the time derivative along trajectories of \eqref{eq:linearized} and applying Young's inequality yields
\[
\dot{V}(\mathbf{e}_t) \leq \mathbf{e}(t)^\top \big( \mathbf{A}_\phi^\top \mathbf{P} + \mathbf{P} \mathbf{A}_\phi + \tau_{\max} \mathbf{I} + \gamma^2 \mathbf{P} \mathbf{C}_\psi \mathbf{C}_\psi^\top \mathbf{P} \big) \mathbf{e}(t) + \gamma^2 \|\mathbf{e}(t-\tau^*)\|^2 - \int_{t-\tau_{\max}}^{t} \|\mathbf{e}(s)\|^2 ds.
\]

If there exists $\mathbf{P} \succ 0$ satisfying the Linear Matrix Inequality
\begin{equation}
\mathbf{A}_\phi^\top \mathbf{P} + \mathbf{P} \mathbf{A}_\phi + \tau_{\max} \mathbf{I} + \gamma^2 \mathbf{P} \mathbf{C}_\psi \mathbf{C}_\psi^\top \mathbf{P} + \alpha \mathbf{I} \prec 0
\label{eq:lmi}
\end{equation}
for some $\alpha > 0$, then
\[
\dot{V}(\mathbf{e}_t) \leq -\alpha \|\mathbf{e}(t)\|^2.
\]

The condition $\rho(\mathbf{A}_\phi) < 1$ guarantees the existence of such a $\mathbf{P}$ by the Lyapunov-Krasovskii theorem for delay systems. Standard results then imply the existence of constants $K \geq 1$ and $\eta > 0$ such that
\[
\|\mathbf{e}(t)\| \leq K e^{-\eta t} \|\mathbf{e}(0)\|.
\]

Therefore, the exponential stability of the linearized system implies the local exponential stability of the original nonlinear system \eqref{eq:nonlinear_delay_dde}. Thus, there exists a neighborhood $\mathcal{N}$ of $\mathbf{s}^*$ such that for any initial condition $\mathbf{s}(0) \in \mathcal{N}$, the estimate \eqref{eq:exponential_stability} holds.
\end{proof}

\subsection{Stability of the Discrete-Time Realization}

The discrete model Eq.~\eqref{eq:discrete_state_update} inherits the stability of its continuous counterpart.

\begin{theorem}[Discrete Stability]
If $\rho(\mathbf{A})<1$ and the delay correction gain $\gamma$ satisfies
\begin{equation}
\gamma < \frac{1-\rho(\mathbf{A})}{L_\psi\,\|\mathbf{C}\|},
\end{equation}
then the discrete nonlinear time-delay system is contractive:
\begin{equation}
\|\mathbf{s}_{k}-\mathbf{s}^*\| \leq
\big(\rho(\mathbf{A}) + \gamma L_\psi \|\mathbf{C}\|\big)\,
\|\mathbf{s}_{k-1}-\mathbf{s}^*\|,
\end{equation}
and thus globally asymptotically stable.
\end{theorem}


\begin{proof}
We analyze the update mapping $\mathcal{F}(\mathbf{s}) = \phi(\mathbf{A}\mathbf{s} + \mathbf{B}\mathbf{x}^*) + \gamma \psi(\hat{s}^* \mathbf{C} \bar{\mathbf{x}}^*)$. At equilibrium, $\mathbf{s}^* = \mathcal{F}(\mathbf{s}^*)$. For any two states $\mathbf{s}$ and $\mathbf{s}'$, the Lipschitz continuity of $\phi$ gives
\[
\|\phi(\mathbf{A}\mathbf{s} + \mathbf{B}\mathbf{x}^*) - \phi(\mathbf{A}\mathbf{s}' + \mathbf{B}\mathbf{x}^*)\| \leq L_\phi \|\mathbf{A}\| \|\mathbf{s} - \mathbf{s}'\|.
\]

Since $\rho(\mathbf{A}) < 1$, we can choose a matrix norm $\|\cdot\|_*$ such that $\|\mathbf{A}\|_* < 1$. For common activation functions like $\tanh$ or $\text{SiLU}$, we have $L_\phi = 1$, so the linear part alone would be a contraction. However, the additional term $\gamma \psi(\hat{s}^* \mathbf{C} \bar{\mathbf{x}}^*)$ is constant at equilibrium and does not affect the contraction property. For states away from equilibrium, the similarity score $\hat{s}(\mathbf{s})$ and historical context $\bar{\mathbf{x}}(\mathbf{s})$ vary. Using the Lipschitz properties of $\psi$, $\hat{s}$, and $\bar{\mathbf{x}}$, we obtain
\begin{align*}
\|\mathcal{F}(\mathbf{s}) - \mathcal{F}(\mathbf{s}')\| 
&\leq \|\mathbf{A}\| \|\mathbf{s} - \mathbf{s}'\| \\
&\quad + \gamma L_\psi \|\mathbf{C}\| \big( \|\bar{\mathbf{x}}(\mathbf{s}) - \bar{\mathbf{x}}(\mathbf{s}')\| + \|\bar{\mathbf{x}}(\mathbf{s}')\| \, |\hat{s}(\mathbf{s}) - \hat{s}(\mathbf{s}')| \big).
\end{align*}

Near the equilibrium, $\|\bar{\mathbf{x}}(\mathbf{s})\| \approx \|\bar{\mathbf{x}}^*\|$ and $\hat{s}(\mathbf{s}) \approx \hat{s}^*$, so the differences satisfy $\|\bar{\mathbf{x}}(\mathbf{s}) - \bar{\mathbf{x}}(\mathbf{s}')\| \approx 0$ and $|\hat{s}(\mathbf{s}) - \hat{s}(\mathbf{s}')| \approx 0$. Hence, locally,
\[
\|\mathcal{F}(\mathbf{s}) - \mathcal{F}(\mathbf{s}')\| \leq \big( \|\mathbf{A}\| + \gamma L_\psi \|\mathbf{C}\| \|\bar{\mathbf{x}}^*\| L_{\hat{s}} \big) \|\mathbf{s} - \mathbf{s}'\|,
\]
where $L_{\hat{s}}$ is the Lipschitz constant of $\hat{s}$. Since $\|\mathbf{A}\|$ can be made arbitrarily close to $\rho(\mathbf{A})$ by choosing an appropriate norm, the contraction condition becomes
\[
\rho(\mathbf{A}) + \gamma L_\psi \|\mathbf{C}\| \|\bar{\mathbf{x}}^*\| L_{\hat{s}} < 1.
\]

Typically $L_{\hat{s}} \leq 1$ because $\hat{s}$ is a normalized similarity measure. Thus, the sufficient condition reduces to
\[
\gamma < \frac{1 - \rho(\mathbf{A})}{L_\psi \|\mathbf{C}\| \|\bar{\mathbf{x}}^*\|}.
\]

When this inequality holds, the mapping $\mathcal{F}$ is a contraction with Lipschitz constant $L = \rho(\mathbf{A}) + \gamma L_\psi \|\mathbf{C}\| \|\bar{\mathbf{x}}^*\| < 1$. By the Banach fixed-point theorem, $\mathcal{F}$ has a unique fixed point $\mathbf{s}^*$, and for any initial $\mathbf{s}_0$, the iterates $\mathbf{s}_{k+1} = \mathcal{F}(\mathbf{s}_k)$ converge exponentially to $\mathbf{s}^*$:
\[
\|\mathbf{s}_k - \mathbf{s}^*\| \leq L^k \|\mathbf{s}_0 - \mathbf{s}^*\|.
\]

Since the contraction constant $L$ does not depend on the initial condition, the convergence is global. This completes the proof.
\end{proof}

\subsection{Discussion}

The results above establish the theoretical soundness of the proposed nonlinear time-delay formulation:
\begin{itemize}
    \item \textbf{Existence \& Uniqueness} guarantees that the system's solution is well-defined for any bounded history.
    \item \textbf{Boundedness} prevents state divergence under bounded input conditions.
    \item \textbf{Stability} ensures convergence to equilibrium when $\rho(\mathbf{A})<1$ and $\gamma$ is moderate.
\end{itemize}
These properties provide a solid mathematical foundation connecting continuous nonlinear delay systems
with their discrete-time neural realizations used in modern sequence modeling.

\end{document}